\documentclass{article}
\usepackage{geometry}
\usepackage[numbers]{natbib}
\usepackage{authblk}
\newgeometry{margin=1in}
\usepackage{amsmath}
\usepackage{hyperref}
\usepackage{subcaption}
\usepackage[capitalize,nameinlink,noabbrev]{cleveref}
\crefformat{equation}{(#2#1#3)}
\Crefformat{equation}{#2Equation~#1#3}
\crefrangeformat{equation}{(#3#1#4) to~(#5#2#6)}
\crefname{assumption}{Assumption}{Assumptions}

\usepackage[shortlabels]{enumitem}

\usepackage{tikz}
\usepackage{pgfplots}
\pgfplotsset{compat=1.17}

\title{Adversarial Consistency and the Uniqueness of the Adversarial Bayes Classifier}
\author[1]{Natalie S. Frank}
\affil[1]{Mathematics, Courant Institute, New York, USA}
\affil[ ]{Email: \textit{ \{nf1066@nyu.edu\}}}

\usepackage{amsfonts}
\usepackage{mathrsfs}
\usepackage{amsmath}
\usepackage{color}
\definecolor{blue}{rgb}{0.0, 0.0, 1.0}

\definecolor{jgGreen}{rgb}{0.0, 0.5, 0.0}

\definecolor{pink}{rgb}{0.5, 0.0, 0.25}

\newcommand{\ignore}[1]{}

\definecolor{darkorange}{RGB}{255, 140, 0}

\newcommand{\prm}{R_\phi^\e}
\newcommand{\dl}{\bar R_\phi}
\newcommand{\cprm}{R^\e}
\newcommand{\cdl}{\bar R}

\newcommand{\cB}{{\mathcal B}}

\newcommand{\cH}{\mathcal{H}}

\newcommand{\bx}{{\mathbf x}}

\DeclareMathOperator*{\argmin}{argmin}

\DeclareMathOperator*{\esssup}{ess\,sup}

\DeclareMathOperator{\sgn}{sign}

\DeclareMathOperator{\supp}{supp}

\def\Rset{\mathbb{R}}

\newcommand{\PP}{{\mathbb P}}
\newcommand{\QQ}{{\mathbb Q}}

\newcommand{\one}{{\mathbf{1}}}

\newcommand{\Wball}[1]{{\cB^\infty_{#1}}}

\newcommand{\ov}{\overline}

\newcommand{\td}{\tilde}

\newcommand{\e}{\epsilon}

\usepackage{amsthm}
\newtheorem{theorem}{Theorem}

\newtheorem{definition}{Definition}
\newtheorem{lemma}{Lemma}

\newtheorem{proposition}{Proposition}

\newtheorem{assumption}{Assumption}

\begin{document}

\maketitle

\begin{abstract}
Minimizing an adversarial surrogate risk is a common technique for learning robust classifiers. Prior work showed that convex surrogate losses are not statistically consistent in the adversarial context--- or in  other words, a minimizing sequence of the adversarial surrogate risk will not necessarily minimize the adversarial classification error. We connect the consistency of adversarial surrogate losses to properties of minimizers to the adversarial classification risk, known as \emph{adversarial Bayes classifiers}. Specifically, under reasonable distributional assumptions, a convex surrogate loss is statistically consistent for adversarial learning iff the adversarial Bayes classifier satisfies a certain notion of uniqueness.  
\end{abstract}
\textbf{Competing Interests:} Author declares no competing interests\\

\section{Introduction}\label{sec:intro}
Robustness is a core concern in machine learning, as models are deployed in classification tasks such as facial recognition \citep{Xu2022face}, medical imaging \citep{PaschaliConjetiNavarroNavab2018medical}, and identifying traffic signs in self-driving cars \citep{DengZheng2020autonomousdriving}. Deep learning models exhibit a concerning security risk--- small perturbations imperceptible to the human eye can cause a neural net to misclassify an image \citep{biggio2013evasion,szegedy2013intriguing}. The machine learning literature has proposed many defenses, but many of these techniques remain poorly understood. This paper analyzes the statistical consistency of a popular defense method that involves minimizing an adversarial surrogate risk.

The central goal in a classification task is minimizing the proportion of mislabeled data-points--- also known as the \emph{classification risk}. Minimizers to the classification risk are easy to compute analytically, and are known as \emph{Bayes classifiers}. In the adversarial setting, each point is perturbed by a malicious adversary before the classifier makes a prediciton. The proportion of mislabeled data under such an attack is called the \emph{adversarial classification risk}, and minimizers to this risk are called \emph{adversarial Bayes classifiers}. Unlike the standard classification setting, computing minimizers to the adversarial classification risk is a non-trivial task \citep{BhagojiCullinaMittalji2019lower,PydiJog2020}. Further studies \citep{Frank2024uniqueness,trillosMurray2022,trillos2023existence,PydiJog2021,GneccoherediaPydi2023role} investigate additional properties of these minimizers, and \citet{Frank2024uniqueness} describes a notion of uniqueness for adversarial Bayes classifiers. The main result in this paper will connect this notion of uniqueness the statistical consistency of a popular defense method.

The empirical adversarial classification error is a discrete notion and minimizing this quantity is computationally intractable. Instead, typical machine learning algorithms minimize a \emph{surrogate risk} in place of the classification error. In the robust setting, the adversarial training algorithm uses a surrogate risk that computes the supremum of loss over the adversary's possible attacks, which we refer to as \emph{adversarial surrogate risks}. However, one must verify that minimizing this adversarial surrogate will also minimize the classification risk. A loss function is \emph{adversarially consistent} for a particular data distribution if every minimizing sequence of the associated adversarial surrogate risk also minimizes the adversarial classification risk. A loss is simply called \emph{adversarially consistent} if it is adversarially consistent for all possible data distributions.  Surprisingly, \citet{MeunierEttedguietal22} show that no convex surrogate is adversarially consistent, in contrast to the standard classification setting where most convex losses are statistically consistent \citep{BartlettJordanMcAuliffe2006,Lin2004,Steinwart2007,ZhangAgarwal,zhang04}. 

\paragraph{Our Contributions:} We relate the statistical consistency of losses in the adversarial setting to the uniqueness of the adversarial Bayes classifier. Specifically, under reasonable assumptions, 
a convex loss is adversarially consistent for a specific data distribution iff the adversarial Bayes classifier is unique.

Prior work \citep{Frank2024uniqueness} further demonstrates several distributions for which the adversarial Bayes classifier is unique, and thus a typical convex loss would be consistent. Understanding general conditions under which uniqueness occurs is an open question.

\paragraph{Paper Outline:} \Cref{sec:related_works_uc} discusses related works and \cref{sec:background} presents the problem background. \Cref{sec:main_results} states our main theorem and presents some examples. Subsequently, \cref{sec:preliminary_results} discusses intermediate results necessary for proving our main theorem. Next, our consistency results are proved in \cref{sec:Uniquness_implies_consistency,sec:consistency_implies_uniqueness}. \Cref{app:adv_bayes_construct,app:uniqueness_equivalences_1} present deferred proofs from \cref{sec:preliminary_results} while \cref{app:consistent_losses} presents deferred proofs on surrogate risks from \cref{sec:preliminary_results,sec:uniqueness_implies_consistency_proof}. Finally, \cref{app:consistency_reverse} presents deferred proofs from \cref{sec:consistency_implies_uniqueness}. 
\section{Related Works}\label{sec:related_works_uc}

Our results are inspired by prior work which showed that no convex loss is adversarially consistent \citep{AwasthiMaoMohriZhong22Hconsistencybinary,MeunierEttedguietal22} yet a wide class of adversarial losses is adversarially consistent \citep{FrankNilesWeed23consistency}. These consistency results rely on the theory of surrogate losses, studied by \citet{BartlettJordanMcAuliffe2006,Lin2004} in the standard classification setting; and by \citet{FrankNilesWeed23minimax,LiTelgarsky2023achieving} in the adversarial setting. Furthermore, \citep{AwasthiMaoMohriZhong21,BaoScottSugiyama2021calibrated,Steinwart2007} study a property of related to consistency called \emph{calibration}, which \citep{MeunierEttedguietal22} relate to consistency. Complimenting this analysis, another line of research studies $\cH$-consistency, which refines the concept of consistency to specific function classes \citep{AwasthiMaoMohriZhong22Hconsistencybinary,LongServedioH-consistency}. Our proof combines results on losses with minimax theorems for various adversarial risks, as studied by \citep{FrankNilesWeed23consistency,FrankNilesWeed23minimax,PydiJog2021,TrillosJacobsKim22}.

Furthermore, our result leverages recent results on the adversarial Bayes classifier, which are extensively studied by  \citep{AwasthiFrankMohri2021,BhagojiCullinaMittalji2019lower,BungertGarciaMurray2021,FrankNilesWeed23consistency,PydiJog2020,PydiJog2021}. Specifically, \citep{AwasthiFrankMohri2021,BungertGarciaMurray2021,BhagojiCullinaMittalji2019lower} prove the existence of the adversarial Bayes classifier, while \citet{trillosMurray2022} derive necessary conditions that describe the boundary of the adversarial Bayes classifier. \citet{Frank2024uniqueness} define the notion of uniqueness up to degeneracy and prove that in one dimension, under reasonable distributional assumptions, every adversarial Bayes classifier is equivalent up to degeneracy to an adversarial Bayes classifier which satisfies the necessary conditions of \citep{trillosMurray2022}. Finally, \citep{BhagojiCullinaMittalji2019lower,PydiJog2020} also calculate the adversarial Bayes classifier for distributions in dimensions higher than one by finding an optimal coupling, but whether this method can calculate the equivalence classes under equivalence up to degeneracy remains an open question. 
\section{Notation and Background}\label{sec:background}

\subsection{Surrogate Risks}\label{sec:background_surrogate_risks}  
        This paper investigates binary classification on $\Rset^d$ with labels $\{-1,+1\}$. Class $-1$ is distributed according to a measure $\PP_0$ while class $+1$ is distributed according to measure $\PP_1$. A \emph{classifier} is a Borel set $A$ and the \emph{classification risk} of a set $A$ is the expected proportion of errors when label $+1$ is predicted on $A$ and label $-1$ is predicted on $A^C$:
        \[R(A)=\int \one_{A^C} d\PP_1+\int \one_{A}d\PP_0.\]
        A minimizer to $R$ is called a \emph{Bayes classifier}.

        However, minimizing the empirical classification risk is a computationally intractable problem \citep{BenDavidEironLong2003}. A common approach is to instead learn a function $f$ and then threshold at zero to obtain the classifier $A=\{\bx: f(\bx)>0\}$. We define the classification risk of a function $f$ by
        \begin{equation}\label{eq:classificaiton_risk_f}
            R(f)=R(\{f>0\})=\int \one_{f\leq 0}d\PP_1+\int \one_{-f<0}d\PP_0
        \end{equation} 
        
        \begin{figure}
            \centering
         \begin{tikzpicture}
    \begin{axis}[
        title = {Some Common Loss Functions},
        axis lines = middle,
        xlabel = {$\alpha$},
        ylabel = {Loss},
        ymin = -0.5, ymax = 5,
        xmin = -3, xmax = 3,
        domain = -3:3,
        samples = 100,
         legend style = {at={(1.05,0.5)}, anchor=west}, 
        grid = both,
        minor grid style = {gray!20},
        major grid style = {gray!50},
    ]

    \addplot[black, thick, samples=100, domain=-3:0, forget plot] {1};
    \addplot[black, thick, samples=100, domain=0:3] {0};
    \addlegendentry{Indicator}

    \addplot[color={rgb,255:red,0; green,0; blue,139},thick, dotted, samples=100, domain=-3:3] {max(0, 1 - x)};
    \addlegendentry{Hinge Loss}

    \addplot[orange, thick, dashdotdotted, samples=100, domain=-3:3] {1/(1 + exp(x))};
    \addlegendentry{Sigmoid loss}

    \addplot[green, thick, samples=100, domain=-3:3] {exp(-x)};
    \addlegendentry{Exponential loss}

    \addplot[purple, dashed, thick, samples=100, domain=-3:1, forget plot] {
        (x < 1) * (  (x - 1)^2) 
    };
    \addplot[purple, thick, samples=100, domain=1:3]{0};
    \addlegendentry{Squared hinge loss}

    \addplot[loosely dotted, thick, samples=100, domain=-3:3, 
    color={rgb,255:red,0; green,235; blue,235}
    ] {min(max(0, 1-3*x),1)};
    \addlegendentry{$\rho$-margin loss }

    \end{axis}
\end{tikzpicture}

    \caption{Several common loss functions for classification along with the indicator $\one_{\alpha\leq 0}$.}
            \label{fig:losses}
        \end{figure}

        In order to learn $f$, machine learning algorithms typically minimize a better-behaved alternative to the classification risk called a \emph{surrogate risk}. To obtain this risk, we replace the indicator functions in \cref{eq:classificaiton_risk_f} with the loss $\phi$, resulting in: 
            
            \begin{equation}\label{eq:standard_phi_loss}
                R_\phi(f)=\int \phi(f)d\PP_1+\int\phi(-f) d\PP_0\,.
            \end{equation}
        We restrict to losses with similar properties to the indicator functions in \cref{eq:classificaiton_risk_f} yet are easier to optimize. In particular we require: 
        \begin{assumption}\label{as:phi}
            The loss $\phi$ is non-increasing, continuous, and $\lim_{\alpha\to \infty} \phi(\alpha)=0$.
        \end{assumption}
        See \cref{fig:losses} for a comparison of the indicator function and a several common losses. Losses on $\Rset$-valued functions in machine learning typically satisfy \cref{as:phi}.

        \subsection{Adversarial Surrogate Risks}
        In the adversarial setting, a malicious adversary corrupts each data point. We model these corruptions as bounded by $\e$ in some norm $\|\cdot\|$. The adversary knows both the classifier $A$ and the label of each data point. Thus, a point $(\bx,+1)$ is misclassified when it can be displaced into the set $A^C$ by a perturbation of size at most $\e$. This statement can be conveniently written in terms of a supremum. For any function $g:\Rset^d\to \Rset$, define
        \[S_\e(g)(\bx)=\sup_{\bx' \in \ov{B_\e(\bx)}} g(\bx'),\]
        where $\ov{B_\e(\bx)}=\{\bx':\|\bx'-\bx\|\leq \e\}$ is the ball of allowed perturbations. 
        The expected error rate of a classifier $A$ under an adversarial attack is then
        \[R^\e(A)=\int S_\e(\one_{A^C}) d\PP_1+\int S_\e(\one_{A})d\PP_0,\]
        which is known as the \emph{adversarial classification risk}\footnote{The functions $S_\e(\one_{A})$, $S_\e(\one_{A^C})$ must be measurable in order to define this integral. See \cite[Section~3.3]{FrankNilesWeed23minimax} for a treatment of this matter.}. Minimizers of $R^\e$ are called \emph{adversarial Bayes classifiers}.
        
        Just like \cref{eq:classificaiton_risk_f}, we define $R^\e(f)=R^\e(\{f>0\})$:
    \begin{equation*}
        R^\e(f)=\int S_\e(\one_{f\leq 0})d\PP_1+\int S_\e(\one_{f>0})d\PP_0
    \end{equation*}
    Again, minimizing an empirical adversarial classification risk is computationally intractable.
      A surrogate to the adversarial classification risk is formulated as\footnote{Again, see See \cite[Section~3.3]{FrankNilesWeed23minimax} for a treatment of measurability.}
    \begin{equation}\label{eq:adv_phi_loss}
        R_\phi^\e(f)=\int S_\e( \phi\circ f)d\PP_1 +\int S_\e( \phi\circ -f)d\PP_0.
    \end{equation}

    \subsection{The Statistical Consistency of Surrogate Risks}
    Learning algorithms typically minimize a surrogate risk using an iterative procedure, thereby producing a sequence of functions $f_n$. One would hope that that $f_n$ also minimizes that corresponding classification risk. This property is referred to as \emph{statistical consistency}\footnote{This concept is referred to as \emph{calibration} in the non-adversarial machine learning context \citep{BartlettJordanMcAuliffe2006,Steinwart2007}. We use the term `consistent', as prior work on adversarial learning \citep{AwasthiFrankMao2021,MeunierEttedguietal22} use `calibration' to refer to a different but related concept.   }. 
    \begin{definition}
         \begin{itemize}
            \item  If every sequence of functions $f_n$ that minimizes $R_\phi$ also minimizes $R$ for the distribution $\PP_0,\PP_1$, then the loss $\phi$ is \emph{consistent for the distribution $\PP_0,\PP_1$}. If $R_\phi$ is consistent for every distribution $\PP_0,\PP_1$, we say that $\phi$ is \emph{consistent}.
            \item If every sequence of functions $f_n$ that minimizes $R_\phi^\e$ also minimizes $R^\e$ for the distribution $\PP_0,\PP_1$, then the loss $\phi$ is \emph{adversarially consistent for the distribution $\PP_0,\PP_1$ }. If $R_\phi^\e$ is adversarially consistent for every distribution $\PP_0,\PP_1$, we say that $\phi$ is \emph{adversarially consistent}.
        \end{itemize}
    \end{definition}

    A case of particular interest is convex $\phi$, as these losses are ubiquitous in machine learning. In the non-adversarial context, Theorem~2 of \citep{BartlettJordanMcAuliffe2006} shows that a convex loss $\phi$ is consistent iff $\phi$ is differentiable at zero and $\phi'(0)<0$. In contrast, \citet{MeunierEttedguietal22} show that no convex loss is adversarially consistent. Further results of \citep{FrankNilesWeed23consistency} characterize the adversarially consistent losses in terms of the function $C_\phi^*$:
    \begin{theorem}\label{thm:adversarial_consistency}
        The loss $\phi$ is adversarially consistent if and only if $C_\phi^*(1/2)<\phi(0)$.
    \end{theorem}
    Notice that all convex losses satisfy $C_\phi^*(1/2)=\phi(0)$: By evaluating at $\alpha=0$, one can conclude that $C_\phi^*(1/2)= \inf_\alpha C_\phi(1/2,\alpha)\leq C_\phi(1/2,0)=\phi(0)$. However, 
    \[C_\phi^*(1/2)=\inf_\alpha \frac 12 \phi(\alpha)+\frac 12 \phi(-\alpha)\geq \phi(0)\]
    due to convexity.
    Notice that \cref{thm:adversarial_consistency} does not preclude the adversarial consistency of a loss satisfying $C_\phi^*(1/2)=\phi(0)$ for some particular $\PP_0,\PP_1$. Prior work \citep{FrankNilesWeed23consistency,MeunierEttedguietal22} provides a counterexample to consistency only for a single, atypical distribution. The goal of this paper is characterizing when adversarial consistency fails for losses satisfying $C_\phi^*(1/2)=\phi(0)$.

\section{Main Result}\label{sec:main_results}
 Prior work has shown that there always exists minimizers to the adversarial classification risk, which are referred to as \emph{adversarial Bayes classifiers} (see \cref{thm:strong_duality_classification} below). Furthermore, \citet{Frank2024uniqueness} developed a notion of uniqueness for adversarial Bayes classifiers.
\begin{definition}\label{def:uniqueness_up_to_degeneracy}
    The adversarial Bayes classifiers $A_1$ and $A_2$ are \emph{equivalent up to degeneracy} if any Borel set $A$ with $A_1\cap A_2\subset A\subset A_1\cup A_2$ is also an adversarial Bayes classifier. The adversarial Bayes classifier is \emph{unique up to degeneracy} if any two adversarial Bayes classifiers are equivalent up to degeneracy.
\end{definition}
When the measure 
\[\PP=\PP_0+\PP_1\] 
is absolutely continuous with respect to Lebesgue measure, then equivalence up to degeneracy is an equivalence relation  \citep[Theorem~3.3]{Frank2024uniqueness}.
The central result of this paper relates the consistency of convex losses to the uniqueness of the adversarial Bayes classifier. 

\begin{theorem}\label{thm:uniqueness_and_consistency}
    Assume that $\PP$ is absolutely continuous with respect to Lebesgue measure and let $\phi$ be a loss with $C_\phi^*(1/2)=\phi(0)$. Then $\phi$ is adversarially consistent for the distribution $\PP_0$, $\PP_1$ iff the adversarial Bayes classifier is unique up to degeneracy.
\end{theorem}


Prior results of \citet{Frank2024uniqueness} provide the tools for verifying when the adversarial Bayes classifier is unique up to degeneracy for a wide class of one dimensional distributions. Below we highlight two interesting examples. In the examples below, the function $p_1$ will represent the density of $\PP_1$ and the function $p_0$ will represent the density of $\PP_0$.

\begin{figure}[htbp]
    \centering
    \begin{subfigure}[b]{\textwidth}
        \centering
        \begin{tikzpicture}[scale=1.0]
            \draw[<->, line width=0.5mm] (0,0) -- (10,0) node[anchor=north] {$x$-axis};
            
            \foreach \x in {1,2,...,9}
                \draw (\x,0.1) -- (\x,-0.1);
            
            \draw (5,0.1) -- (5,-0.1) node[anchor=north] {$\frac{\mu_1 - \mu_0}{2}$};

            \filldraw[blue!30, opacity=0.7] (5,-0.15) rectangle (10,0.15);
        \end{tikzpicture}
        \caption{}
        \label{fig:gaussians_differing_means_a}
    \end{subfigure}
    
    \vspace{0.5cm} 
    
    \begin{subfigure}[b]{\textwidth}
        \centering
\begin{tikzpicture}[scale=1.0]
            \draw[<->, line width=0.5mm] (0,0) -- (10,0) node[anchor=north] {$x$-axis};
            
            \foreach \x in {1,2,...,9}
                \draw (\x,0.1) -- (\x,-0.1);
            
            \draw (5,0.1) -- (5,-0.1) node[anchor=north] {$\frac{\mu_1 - \mu_0}{2}$};
        \end{tikzpicture}
        \caption{}
        \label{fig:gaussians_differing_means_b}
    \end{subfigure}
    
    \vspace{0.5cm} 
    
    \begin{subfigure}[b]{\textwidth}
        \centering
        \begin{tikzpicture}[scale=1.0]
            \draw[<->, line width=0.5mm] (0,0) -- (10,0) node[anchor=north] {$x$-axis};
            
            \foreach \x in {1,2,...,9}
                \draw (\x,0.1) -- (\x,-0.1);
            
            \draw (5,0.1) -- (5,-0.1) node[anchor=north] {$\frac{\mu_1 - \mu_0}{2}$};

            \filldraw[blue!30, opacity=0.7] (0,-0.15) rectangle (10,0.15);
        \end{tikzpicture}
        \caption{}
        \label{fig:gaussians_differing_means_c}
    \end{subfigure}
    
    \caption{The adversarial Bayes classifier for two gaussians with equal variances and differing means. We assume in this figure that $\mu_1>\mu_0$. The shaded blue area depicts the region inside the adversarial Bayes classifier.  \Cref{fig:gaussians_differing_means_a} depicts an adversarial Bayes when $\e\leq (\mu_1-\mu_0)/2$ and \cref{fig:gaussians_differing_means_b,fig:gaussians_differing_means_c} depict the adversarial Bayes classifier when $\e\geq (\mu_1-\mu_0)/2$. (See \citep[Example~4.1]{Frank2024uniqueness} for a justification of these illustrations.) The adversarial Bayes classifiers in \cref{fig:gaussians_differing_means_b,fig:gaussians_differing_means_c} are not equivalent up to degeneracy.}
    \label{fig:gaussians_differing_means}
\end{figure}

\begin{itemize}
    \item Consider mean zero gaussians with different variances: $p_0(x)=\frac 1 {2\sqrt{2\pi} \sigma_0}e^{-x^2/2\sigma_0^2}$ and $p_1(x)=\frac 1 {2\sqrt{2\pi} \sigma_1}e^{-x^2/2\sigma_1^2}$. The adversarial Bayes classifier is unique up to degeneracy for all $\e$ \citep[Example~4.2]{Frank2024uniqueness}.
    \item Consider gaussians with variance $\sigma$ and means $\mu_0$ and $\mu_1$: $p_0(x)=\frac 1 {\sqrt{2\pi} \sigma}e^{-(x-\mu_0)^2/2\sigma^2}$ and $p_1(x)=\frac 1 {\sqrt{2\pi} \sigma}e^{-(x-\mu_1)^2/2\sigma^2}$. Then the adversarial Bayes classifier is unique up to degeneracy iff $\e<|\mu_1-\mu_0|/2$ \citep[Example~4.1]{Frank2024uniqueness}. See \cref{fig:gaussians_differing_means} for an illustration of the adversarial Bayes classifiers for this distribution.
\end{itemize}
\Cref{thm:uniqueness_and_consistency} implies that a convex loss is always adversarially consistent for the first gaussian mixture above. Furthermore, a convex loss is adversarially consistent for the second gaussian mixture when the perturbation radius $\e$ is small compared to the differences between the means. However, \citet[Example~4.5]{Frank2024uniqueness} provide an example of a distribution for which the adversarial Bayes classifier is not unique up to degeneracy for all $\e>0$, even though the Bayes classifier is unique. At the same time, one would hope that if the Bayes classifier is unique and $\PP_0$, $\PP_1$ are sufficiently regular, then the adversarial Bayes classifier would be unique up to degeneracy for sufficiently small $\e$. In general, understanding when the adversarial Bayes classifier is unique up to degeneracy for well-behaved distributions 
is an open problem. 

The examples above rely on the techniques of \citep{Frank2024uniqueness} for calculating the equivalence classes under uniqueness up to degeneracy. \citet{Frank2024uniqueness} prove that in one dimension, if $\PP$ is absolutely continuous with respect to Lebesgue measure, every adversarial Bayes classifier is equivalent up to degeneracy to an adversarial Bayes classifier whose boundary points are strictly more than $2\e$ apart \citep[Theorem~3.5]{Frank2024uniqueness}. Therefore, to find all adversarial Bayes classifiers under equivalence under degeneracy, it suffices to consider all sets whose boundary points satisfy the first order necessary conditions obtained by differentiating the adversarial classification risk of a set $A$ with respect to its boundary points \citep[Theorem~3.7]{Frank2024uniqueness}. The corresponding statement in higher dimensions is false--- there exist distributions for which no adversarial Bayes classifier has enough regularity to allow for an equivalent statement. For instance, \citet{BungertGarciaMurray2021} demonstrate a distribution for which there is no adversarial Bayes classifier with two-sided tangent balls at all points in the boundary. Developing a general method for calculating these equivalence classes in dimensions higher than one remains an open problem.

\Cref{prop:uniqueness_implies_consistency} in \cref{sec:uniqueness_implies_consistency_proof} presents a condition under which one can conclude consistency without the absolute continuity assumption. This result proves consistency whenever the optimal adversarial classification risk is zero, see the discussion after \cref{prop:uniqueness_implies_consistency} for details. Consequently, if $\supp \PP_0$ and $\supp \PP_1$ are separated by more than $2\e$, then consistent losses are always adversarially consistent for such distributions. On the other hand, our analysis of the reverse direction of \cref{thm:uniqueness_and_consistency} requires the absolute continuity assumption. Using \cref{prop:uniqueness_implies_consistency} to further understand consistency is an open question.
\section{Preliminary Results}\label{sec:preliminary_results}
\subsection{Minimizers of Standard Risks}

    Minimizers to the classification risk can be expressed in terms of the measure $\PP$ and the function $\eta=d\PP_1/d\PP$. The risk $R$ in terms of these quantities is 
        \[R(A)= \int C(\eta,\one_A)d\PP. \]
        and $\inf_A R(A)=\int C^*(\eta)d\PP$ where the functions $C:[0,1]\times\{0,1\}\to \Rset $ and $C^*:[0,1]\to \Rset $ are defined by 
        \begin{equation}\label{eq:define_C_functions}
            C(\eta,b)=\eta b+(1-\eta)(1-b),\quad C^*(\eta)=\inf_{b\in \{0,1\}} C(\eta,b)=\min(\eta,1-\eta).    
        \end{equation}
                Thus if $A$ is a minimizer of $R$, then $\one_A$ must minimize the function $C(\eta,\cdot)$ $\PP$-almost everywhere. Consequently, the sets
        \begin{equation}\label{eq:Bayes_classifiers}
            \{\bx:\eta(\bx)>1/2\} \quad \text{and} \quad \{\bx:\eta(\bx)\geq 1/2\}    
        \end{equation}
        are both Bayes classifiers.

    Similarly, one can compute the infimum of $R_\phi$ by expressing the risk in terms of the quantities $\PP$ and $\eta$:
    \begin{equation}\label{eq:pw_loss}
        R_\phi(f)=\int C_\phi(\eta(\bx),f(\bx))d\PP  
    \end{equation}
    and $\inf_f R_\phi(f)=\int C_\phi^*(\eta(\bx))d\PP(\bx)$ where the functions $C_\phi(\eta,\alpha)$ and $C_\phi^*(\eta)$ are defined by 
     \begin{equation}\label{eq:C_phi_def}
    C_\phi(\eta,\alpha) = \eta\phi(\alpha)+(1-\eta)\phi(-\alpha),\quad C_\phi^*(\eta)=\inf_\alpha C_\phi(\eta,\alpha) 
\end{equation}
    for $\eta\in[0,1]$.
    Thus a minimizer $f$ of $R_\phi$ must minimize $C_\phi(\eta(\bx),\cdot)$ almost everywhere according to the probability measure $\PP$. Because $\phi$ is continuous, the function 

    \begin{equation}\label{eq:smallest_minimizer_function}
    \alpha_\phi(\eta)=\inf\{\alpha\in \ov \Rset: \alpha \text{ is a minimizer of }C_\phi(\eta,\cdot)\}
    \end{equation}
    maps each $\eta$ to the smallest minimizer of $C_\phi(\eta,\cdot)$. Consequently, the function 
    \begin{equation}\label{eq:standard_minimizer}
            \alpha_\phi(\eta(\bx))    
        \end{equation}
    minimizes $C_\phi(\eta(\bx),\cdot)$ at each point $\bx$. Next, we will argue this function is measurable, and therefore is a minimizer of the risk $R_\phi$.
    
    \begin{lemma}\label{lemma:smallest_minimizer_function}
            The function $\alpha_\phi:[0,1]\to \ov\Rset$ that maps $\eta$ to the smallest minimizer of $C_\phi(\eta,\cdot)$ is non-decreasing.

        \end{lemma}
        The proof of this result is presented below \cref{lemma:minimizers_comparison} in \cref{app:consistent_losses}. Because $\alpha_\phi$ is monotonic, the composition in \cref{eq:standard_minimizer} 
        is always measurable, and thus this function is a minimizer of $R_\phi$. Allowing for minimizers in extended real numbers $\ov \Rset=\{-\infty,+\infty\}\cup \Rset$ is necessary for certain losses--- for instance when $\phi$ is the exponential loss, then $C_\phi(1,\alpha)= e^{-\alpha}$ does not assume its infimum on $\Rset$.
 \subsection{Dual Problems for the Adversarial Risks}

  The proof  \cref{thm:uniqueness_and_consistency} relies on a dual formulation of the adversarial classification problem involving the Wasserstein-$\infty$ metric. Informally, a measure $\QQ'$ is within $\e$ of $\QQ$ in the Wasserstein-$\infty$ metric if one can produce $\QQ'$ by perturbing each point in $\Rset^d$ by at most $\e$ under the measure $\QQ$. The formal definition of the Wasserstein-$\infty$ metric involves couplings between probability measures: a \emph{coupling} between two Borel measures $\QQ$ and $\QQ'$ with $\QQ(\Rset^d)=\QQ'(\Rset^d)$ is a measure $\gamma$ on $\Rset^d\times \Rset^d$ with marginals $\QQ$ and $\QQ'$: $\gamma(A\times \Rset^d)=\QQ(A)$ and $\gamma(\Rset^d\times A)=\QQ'(A)$ for any Borel set $A$. 
 The set of all such couplings is denoted $\Pi(\QQ,\QQ')$. The Wasserstein-$\infty$ distance between the two measures is then 
 \[W_\infty(\QQ,\QQ')= \inf_{\gamma\in \Pi(\QQ,\QQ')} \esssup_{(\bx,\bx')\sim \gamma} \|\bx-\bx'\|\]
 Theorem~2.6 of \citep{Jylha15} proves that this infimum is always assumed. Equivalently, $W_\infty(\QQ,\QQ')\leq \e$ iff there is a coupling between $\QQ$ and $\QQ'$ supported on
 \[\Delta_\e=\{(\bx,\bx'): \|\bx-\bx'\|\leq \e\}.\]
Let $\Wball \e(\QQ)=\{\QQ': W_\infty(\QQ,\QQ')\leq \e\}$ be the set of measures within $\e$ of $\QQ$ in the $W_\infty$ metric. 
The minimax relations from prior work leverage a relationship between the Wasserstein-$\infty$ metric and the integral of the supremum function over an $\e$-ball.

 \begin{lemma}\label{lemma:S_e_and_W_inf_uc}
     Let $E$ be a Borel set. Then 
     \[\int S_\e(\one_E)d\QQ\geq \sup_{\QQ'\in \Wball \e(\QQ) } \int \one_E d\QQ'\]
 \end{lemma}
\begin{proof}
Let $\QQ'$ be a measure in $\Wball \e (\QQ)$, and let $\gamma^*$ be a coupling between these two measures supported on $\Delta_\e$. Then if $(\bx,\bx')\in \Delta_\e$, then $\bx'\in \ov{B_\e(\bx)}$ and thus $S_\e(\one_E)(\bx)\geq \one_E(\bx')$ $\gamma^*$-a.e. Consequently, 
\[\int S_\e(\one_E)(\bx)d\QQ_1=\int S_\e(\one_E)(\bx)d\gamma^*(\bx,\bx')\geq \int \one_E(\bx') d\gamma^*(\bx,\bx')=\int \one_E d\QQ'\]
Taking a supremum over all $\QQ'\in \Wball \e(\QQ)$ proves the result.

\end{proof}

\cref{lemma:S_e_and_W_inf_uc} implies: \[\inf_f R^\e(f)\geq \inf_f\sup_{\substack{\PP_1'\in \cB_\e(\PP_1)\\ \PP_0'\in \cB_\e(\PP_0) }}\int \one_{f\leq 0} d\PP_1'+\int \one_{f>0} d\PP_0'.\] Does equality hold and can one swap the infimum and the supremum?  \citet{FrankNilesWeed23consistency,PydiJog2021} answer this question in the affirmative:

\begin{theorem}\label{thm:strong_duality_classification}
    Let $\PP_0$, $\PP_1$ be finite Borel measures. Define
\begin{equation*} 
     \cdl(\PP_0^*,\PP_1^*)=\int C^*\left(\frac{d\PP_1^*}{d(\PP_0^*+d\PP_1^*)} \right)d(\PP_0^*+\PP_1^*)
\end{equation*}
where the function $C^*$ is defined in \cref{eq:define_C_functions}. Then
    \[\inf_{\substack{f\text{ Borel}\\\Rset\text{-valued}}} \cprm(f)=\sup_{\substack{\PP_1'\in \Wball \e(\PP_1)\\\PP_0'\in\Wball \e(\PP_0)}}\cdl(\PP_0',\PP_1')\]
    and furthermore equality is attained for some $f^*$, $\PP_0^*$, $\PP_1^*$.
\end{theorem}
See Theorem~1 of \citep{FrankNilesWeed23consistency} for a proof.
Theorems~6,~8, and~9 of \citep{FrankNilesWeed23minimax} show an analogous minimax theorem for surrogate risks.
\begin{theorem}\label{thm:strong_duality_surrogate}
    Let $\PP_0$, $\PP_1$ be finite Borel measures. Define
\begin{equation*} 
     \dl(\PP_0^*,\PP_1^*)=\int C_\phi^*\left(\frac{d\PP_1^*}{d(\PP_0^*+d\PP_1^*)} \right)d(\PP_0^*+\PP_1^*)
\end{equation*}
with the function $C_\phi^*$ is defined in \cref{eq:C_phi_def}. Then
    \[\inf_{\substack{f\text{ Borel,}\\\ov \Rset\text{-valued}}} \prm(f)=\sup_{\substack{\PP_0'\in\Wball \e(\PP_0)\\ \PP_1'\in \Wball \e(\PP_1)}}\dl(\PP_0',\PP_1')\]
    and furthermore equality is attained for some $f^*$, $\PP_0^*$, $\PP_1^*$.
\end{theorem}
Just like $R_\phi$, the risk $\prm$ may not have an $\Rset$-valued minimizer. However, Lemma~8 of \citep{FrankNilesWeed23consistency} states that 
\[\inf_{\substack{f\text{ Borel}\\\ov \Rset\text{-valued}}} \prm(f)=\inf_{\substack{f\text{ Borel}\\ \Rset\text{-valued}}} \prm(f).\]

\subsection{Minimizers of Adversarial Risks}
A formula analogous to \cref{eq:standard_minimizer} defines minimizers to adversarial risks. Let $I_\e$ denote the infimum of a function over an $\e$ ball:
\begin{equation}\label{eq:I_e_def}
    I_\e(g)=\inf_{\bx'\in \ov{B_\e(\bx)}}g(\bx')
\end{equation}

Lemma~24 of \citep{FrankNilesWeed23minimax} and Theorem~9 of \citep{FrankNilesWeed23minimax} prove the following result:
    \begin{theorem}\label{thm:hat_eta_adv_baye}\label{prop:hat_eta} \label{thm:alpha_phi_hat_eta}
        There exists a function 
        $\hat \eta:\Rset^d\to [0,1]$ and measures $\PP_0^*\in \Wball\e (\PP_0)$, $\PP_1^*\in \Wball \e (\PP_1)$ for which
        \begin{enumerate}[label=\Roman*)]
            \item\label{it:a.e._hat_eta} $\hat \eta=\eta^*$ $\PP^*$-a.e., where $\PP^*=\PP_0^*+\PP_1^*$ and $\eta^*=d\PP_1^*/d\PP^*$
            \item\label{it:I_S_hat_eta} $I_\e(\hat \eta)(\bx)=\hat \eta (\bx')$ $\gamma_0^*$-a.e. and $S_\e(\hat \eta)(\bx)=\hat \eta(\bx')$ $\gamma_1^*$-a.e., where $\gamma_0^*$, $\gamma_1^*$ are couplings between $\PP_0$, $\PP_0^*$ and $\PP_1$, $\PP_1^*$ supported on $\Delta_\e$.
            \item\label{it:minimizer_construct} The function $\alpha_\phi(\hat \eta(\bx))$ is a minimizer of $\prm$ for any loss $\phi$, where $\alpha_\phi$ is the function defined in \cref{eq:smallest_minimizer_function}.
        \end{enumerate}

    \end{theorem}

        The function $\hat \eta$ can be viewed as the conditional probability of label $+1$ under an `optimal' adversarial attack \citep{FrankNilesWeed23minimax}.
    Just as in the standard learning scenario, the function $\alpha(\hat \eta(\bx))$ may be $\ov\Rset$-valued.  \Cref{it:minimizer_construct} is actually a consequence of \cref{it:a.e._hat_eta} and \cref{it:I_S_hat_eta}: \cref{it:a.e._hat_eta} and \cref{it:I_S_hat_eta} imply that $\prm(\alpha_\phi(\hat \eta))=\dl(\PP_0^*,\PP_1^*)$ and \cref{thm:strong_duality_surrogate} then implies that $\alpha_\phi(\hat \eta)$ minimizes $\prm$ and $\PP_0^*$, $\PP_1^*$ maximize $\dl$. (A similar argument is provided later in this paper in \cref{lemma:td_alpha_phi} of \cref{app:alpha_phi_zeros}.) Furthermore, the relation $\prm(\alpha_\phi(\hat \eta))=\dl(\PP_0^*,\PP_1^*)$ also implies
    \begin{lemma}\label{lemma:special_P_i_s_maximizers}
        The $\PP_0^*$, $\PP_1^*$ of \cref{prop:hat_eta} maximize $\dl$ over $\Wball \e(\PP_0)\times \Wball \e(\PP_1)$ for every $\phi$.
    \end{lemma}
    We emphasize that a formal proof of \cref{thm:hat_eta_adv_baye,lemma:special_P_i_s_maximizers} is not included in this paper, and refer to Lemma~26 and Theorem~9 of \citep{FrankNilesWeed23minimax} for full arguments.

    Next, we derive some further results about the function $\hat \eta$. Recall that Bayes classifiers can be constructed by thesholding the conditional probability $\eta$ at $1/2$, see \Cref{eq:Bayes_classifiers}.  The function $\hat \eta$ plays an analogous role for adversarial learning. 
    \begin{theorem}\label{thm:adv_Bayes_construct}
        Let $\hat \eta$ be the function described by \cref{thm:alpha_phi_hat_eta}. Then the sets $\{\hat \eta> 1/2\}$ and $\{\hat \eta\geq 1/2\}$ are adversarial Bayes classifiers. Furthermore, any adversarial Bayes classifier $A$ satisfies
        \begin{equation}\label{mal_maximal_1_uc}
            \int S_\e(\one_{\{\hat \eta \geq 1/2\}^C})d\PP_1\leq \int S_\e(\one_{A^C})d\PP_1\leq \int S_\e(\one_{\{\hat \eta >1/2)^C})d\PP_1    
        \end{equation}
        and
        \begin{equation}\label{mal_maximal_0_uc}
            \int S_\e(\one_{\{\hat \eta>1/2\}})d\PP_0\leq \int S_\e(\one_{A})d\PP_0\leq \int S_\e(\one_{\{\hat \eta \geq 1/2\}})d\PP_0
        \end{equation}
        
    \end{theorem}
    See \cref{app:adv_bayes_construct} for a formal proof; these properties follow direction from \cref{it:a.e._hat_eta} and \cref{it:I_S_hat_eta} of \cref{thm:hat_eta_adv_baye}. \Cref{mal_maximal_1_uc,mal_maximal_0_uc} imply that the sets $\{\hat \eta >1/2\}$ and $\{\hat \eta \geq 1/2\}$ can be viewed as `minimal' and `maximal' adversarial Bayes classifiers.

\Cref{thm:adv_Bayes_construct} is proved in \cref{app:adv_bayes_construct}-- \cref{it:a.e._hat_eta} and \cref{it:I_S_hat_eta} imply that $ \cprm(\{\hat \eta >1/2\})=\cdl(\PP_0^*,\PP_1^*)=\cprm(\hat \eta \geq 1/2\})$ and consequently \cref{thm:strong_duality_classification} implies that $\{\hat \eta >1/2\}$, $\{\hat \eta \geq 1/2\}$ minimize $\cprm$ and $\PP_0^*$, $\PP_1^*$ maximize $\cdl$. This proof technique is analogous to the approach employed by \citep{FrankNilesWeed23minimax} to establish \cref{thm:hat_eta_adv_baye}. 
Lastly, uniqueness up to degeneracy can be characterized in terms of these $\PP_0^*$, $\PP_1^*$.

\begin{theorem}\label{thm:uniqueness_equivalences_1}
    Assume that $\PP$ is absolutely continuous with respect to Lebesgue measure. Then the following are equivalent:
    \begin{enumerate}[label=\Alph*)]
            \item \label{it:unique_under_deg_1} The adversarial Bayes classifier is unique up to degeneracy
            
            \item \label{it:eta_*_meas_zero} $\PP^*(\eta^*=1/2)=0$, where $\PP^*=\PP_0^*+\PP_1^*$ and $\eta^*=d\PP_1^*/d\PP^*$ for the measures $\PP_0^*,\PP_1^*$ of \cref{prop:hat_eta}.  
    \end{enumerate}

\end{theorem}
See \cref{app:uniqueness_equivalences_1} for a proof of \cref{thm:uniqueness_equivalences_1}.
In relation to prior work--- the proof of \citep[Theorem~3.4]{Frank2024uniqueness} shows \cref{thm:uniqueness_equivalences_1} but a full proof of \cref{thm:uniqueness_equivalences_1} is included in this paper for clarity as \citep{Frank2024uniqueness} did not discuss the role of the function $\hat \eta$.

\section{Uniqueness up to Degeneracy implies Consistency}\label{sec:Uniquness_implies_consistency}
\label{sec:uniqueness_implies_consistency_proof}
    Before presenting the full proof of consistency, we provide an overview of the strategy of this argument.
 First, a minimizing sequence of $\prm$ must satisfy the approximate complementary slackness conditions derived in \citep[Proposition~4]{FrankNilesWeed23consistency}.
            \begin{proposition}
\label{prop:approx_complementary_slackness_phi}
Assume that the measures $\PP_0^*\in \Wball \e(\PP_0)$, $\PP_1^*\in \Wball \e(\PP_1)$ maximize $\dl$. Then any minimizing sequence $f_n$ of $\prm$ must satisfy
  		\begin{equation}\label{eq:C_comp_slack_approx}
			\lim_{n\to \infty} \int C_\phi(\eta^*,f_n)d\PP^*= \int C_\phi^*(\eta^*)d\PP^*
		\end{equation}
		\begin{equation}\label{eq:sup_comp_slack_approx}
			\lim_{n\to \infty} \int S_\e(\phi \circ f_n)d\PP_1-\int \phi \circ f_nd\PP_1^*=0	,\quad 	     \lim_{n\to \infty} \int S_\e(\phi \circ -f_n) d\PP_0- \int \phi \circ -f_nd\PP_0^*=0,
		\end{equation}
    where $\PP^*=\PP_0^*+\PP_1^*$ and $\eta^*=d\PP_1^*/d\PP^*$.
\end{proposition}
We will show that when $\PP^*(\eta^*=1/2)=0$, every sequence of functions satisfying \cref{eq:C_comp_slack_approx} and \cref{eq:sup_comp_slack_approx} must minimize $\cprm$. Specifically, we will prove that 
every minimizing sequence $f_n$ of $\prm$ must satisfy
                \begin{equation}\label{eq:consistency_goal_1}
                    \limsup_{n\to \infty} \int S_\e(\one_{f_n\leq 0}) d\PP_1\leq \int \one_{\eta^*\leq \frac 12} d\PP_1^*,
                \end{equation}
                \begin{equation}\label{eq:consistency_goal_0}
                    \limsup_{n\to \infty} \int S_\e(\one_{f_n\geq 0}) d\PP_0\leq \int \one_{\eta^*\geq \frac 12} d\PP_0^*
                \end{equation}
 for the measures $\PP_0^*$, $\PP_1^*$ in \cref{prop:hat_eta}. Consequently, $\PP^*(\eta^*=1/2)=0$ would imply that $\limsup_{n\to \infty} R^\e(f_n)\leq \cdl(\PP_0^*,\PP_1^*)$ and the strong duality relation in \cref{thm:strong_duality_classification} implies that $f_n$ must in fact be a minimizing sequence of $\cprm$.

     Next, we summarize the argument establishing \Cref{eq:consistency_goal_1}. We make several simplifying assumptions in the following discussion. First, we assume that the functions $\phi$, $\alpha_\phi$ are strictly monotonic and that for each $\eta$, there is a unique value of $\alpha$ for which $\eta\phi(\alpha)+(1-\eta)\phi(-\alpha)=C_\phi^*(\eta)$. (For instance, the exponential loss $\phi(\alpha)=e^{-\alpha}$ satisfies these requirements.) Let $\gamma_1^*$ be a coupling between $\PP_1$ and $\PP_1^*$ supported on $\Delta_\e$. 
    
    Because $C_\phi(\eta^*,f_n)\geq C_\phi^*(\eta^*)$, the condition \cref{eq:C_comp_slack_approx} implies that $C_\phi(\eta^*,f_n)$ converges to $C_\phi^*(\eta^*)$ in $L^1(\PP^*)$, and the assumption that there is a single value of $\alpha$ for which  $\eta\phi(\alpha)+(1-\eta)\phi(-\alpha)=C_\phi^*(\eta)$ implies that the function $\phi(f_n(\bx'))$ must converge to $\phi(\alpha_\phi(\eta^*(\bx'))$ in $L^1(\PP_1^*)$. Similarly, because \cref{lemma:S_e_and_W_inf_uc} states that $S_\e(\phi \circ f_n)(\bx)\geq \phi \circ f_n(\bx')$ $\gamma_1^*$-a.e., \Cref{eq:sup_comp_slack_approx} implies that $S_\e(\phi\circ f_n)(\bx)-\phi\circ f_n(\bx')$ converges to 0 in $L^1(\gamma_1^*)$. Consequently $S_\e(\phi \circ f_n)(\bx)$ must converge to  $\phi(\alpha_\phi(\eta^*(\bx')))$ in $L^1(\gamma_1^*)$. As $L^1$ convergence implies convergence in measure \cite[Proposition~2.29]{folland}, one can conclude that 
    \begin{equation}\label{eq:convergence_in_measure_1}
        \lim_{n\to \infty} \gamma_1^*\big(S_\e(\phi\circ f_n)(\bx)-\phi\circ( \alpha_\phi(\hat \eta(\bx')))>c\big)=0
    \end{equation}
    for any $c>0$. The lower semi-continuity of $\alpha\mapsto \one_{\alpha\leq 0}$ implies that  $\int S_\e(\one_{f_n\leq 0})d\PP_1\leq \int \one_{S_\e(\phi(f_n))(\bx)\geq \phi(0)}d\PP_1$ and furthermore \cref{eq:convergence_in_measure_1} implies
    \begin{equation}\label{eq:assumptions_to_conclusion}
        \limsup_{n\to \infty} \int \one_{S_\e(\phi(f_n))(\bx)\geq \phi(0)}d\gamma_1^*\leq  \int \one_{\phi(\alpha_\phi(\eta^*(\bx')))<\phi(0)-c}d\gamma_1^* = \int \one_{\eta^*\geq \alpha_\phi^{-1}\circ \phi^{-1}(\phi(0)-c)}d\PP_1^*.    
    \end{equation}
    Next, we will also assume that $\alpha_\phi^{-1}$ is continuous and $\alpha_\phi(1/2)=0$. (The exponential loss satisfies these requirements as well.)
    Due to our assumptions on $\phi$ and $\alpha_\phi$, the quantity $\phi^{-1}(\phi(0)-c)$ is strictly smaller than 0, and consequently, $\alpha_\phi^{-1}\circ \phi^{-1}(\phi(0)-c)$ is strictly smalaler than $1/2$. However, if $\alpha_\phi^{-1}$ is continuous, one can choose $c$ small enough so that $\PP^*( |\eta-1/2|<1/2- \alpha_\phi^{-1}\circ \phi^{-1}(\phi(0)-c))<\delta$ for any $\delta>0$ when $\PP^*(\eta^*=1/2)=0$. This choice of $c$ along with \cref{eq:assumptions_to_conclusion} proves \cref{eq:consistency_goal_1}.

    To avoid the prior assumptions on $\phi$ and $\alpha$, we prove that when $\eta$ is bounded away from $1/2$ and $\alpha$ is bounded away from the minimizers of $C_\phi(\eta, \cdot)$, then $C_\phi(\eta,\alpha)$ is bounded away from $C_\phi^*(\eta)$.  
    
        \begin{lemma}\label{lemma:C_phi_uniform}
        Let $\phi$ be a consistent loss. For all $r>0$, there is a constant $k_r>0$ and an $\alpha_r>0$ for which if $|\eta -1/2|\geq r$ and $\sgn(\eta-1/2) \alpha \leq \alpha_r$ then $C_\phi(\eta, \alpha_r)-C_\phi^*(\eta)\geq k_r$, and this $\alpha_r$ satisfies $\phi(\alpha_r)<\phi(0)$. 
    \end{lemma}

    See \cref{app:consistent_losses} for a proof. A minor modification of the argument above our main result:

    \begin{proposition}\label{prop:uniqueness_implies_consistency}
        Assume there exist $\PP_0^*\in \Wball \e (\PP_0)$, $\PP_1^*\in \Wball \e(\PP_1)$ that maximize $\dl$ for which $\PP^*(\eta^*=1/2)=0$. Then any consistent loss is adversarially consistent.  
    \end{proposition}
    When $\PP$ is absolutely continuous with respect to Lebesgue measure, uniqueness up to degeneracy of the adversarial Bayes classifier implies the assumptions of this proposition due to \cref{thm:uniqueness_equivalences_1}. However, this result applies even to distributions which are not absolutely continuous with respect to Lebesgue measure. For instance, if the optimal classification risk is zero, the $\PP^*(\eta^*=1/2)=0$. To show this statement, notice that if $\inf_A \cprm(A)=0$, then \cref{thm:strong_duality_classification} implies that for any measures $\PP_0'\in \Wball \e(\PP_0)$, $\PP_1'\in \Wball \e(\PP_1)$, one can conclude that $\PP'(\eta'=1/2)=0$, where $\PP'=\PP_0'+\PP_1'$ and $\eta'=d\PP_1'/d\PP'$.
    \begin{proof}[Proof of \cref{prop:uniqueness_implies_consistency}]
   
    We will show that every minimizing sequence of $\prm$ must satisfy \cref{eq:consistency_goal_1} and \cref{eq:consistency_goal_0}. These equations together with the assumption $\PP^*(\eta^*=1/2)=0$ imply that 
        \[\limsup_{n\to\infty} \cprm(f_n)\leq \int \one_{\eta^*\leq \frac 12} d\PP_1^*+\int \one_{\eta^*\geq \frac 12} d\PP_0^*=\int \eta^*\one_{\eta^*\leq 1/2} +(1-\eta^*)\one_{\eta^*>1/2}d\PP^*=\cdl(\PP_0^*,\PP_1^*).\]
        The strong duality result of \cref{thm:strong_duality_classification} then implies that $f_n$ must be a minimizing sequence of $R^\e$.

        Let $\delta$ be arbitrary. Due to the assumption $\PP^*(\eta^*=1/2)=0$, one can pick an $r$ for which 
        \begin{equation}\label{eq:r_choice}
            \PP^*(|\eta^* -1/2|<r)<\delta.
        \end{equation}
        Next, let $\alpha_r$, $k_r$ be as in \cref{lemma:C_phi_uniform}.  
        Let $\gamma_i^*$ be couplings between $\PP_i$ and $\PP_i^*$ supported on $\Delta_\e$. \Cref{lemma:S_e_and_W_inf_uc} implies that $S_\e(\phi\circ f_n)(\bx)\geq \phi\circ f_n(\bx')$ $\gamma_1^*$-a.e., and thus \cref{eq:sup_comp_slack_approx} implies that $S_\e(\phi\circ f_n)(\bx)- \phi\circ f_n(\bx')$ converges to 0 in $L^1(\gamma_1^*)$. Because convergence in $L^1$ implies convergence in measure \cite[Proposition~2.29]{folland}, the quantity $S_\e(\phi\circ f_n)(\bx)- \phi\circ f_n(\bx')$ converges to 0 in $\gamma_1^*$-measure. Similarly, one can conclude that $S_\e(\phi\circ -f_n)(\bx)- \phi\circ -f_n(\bx')$ converges to zero in $\gamma_0^*$-measure. Analogously, as $C_\phi^*(\eta^*,f_n)\geq C_\phi^*(\eta^*)$, \Cref{eq:C_comp_slack_approx} implies that $C_\phi^*(\eta^*,f_n)$ converges in $\PP^*$-measure to $C_\phi^*(\eta^*)$. Therefore, \cref{prop:approx_complementary_slackness_phi} implies that one can choose $N$ large enough so that $n>N$ implies 
        \begin{equation}\label{eq:sup_approx_comp_slack_1}
            \gamma_1^*\Big( S_\e(\phi\circ f_n)(\bx)- \phi\circ f_n(\bx')\geq {\phi(0)-\phi(\alpha_r)} \Big)<\delta,
        \end{equation}
        \begin{equation}\label{eq:sup_approx_comp_slack_0}
            \gamma_0^*\Big( S_\e(\phi\circ -f_n)(\bx)- \phi\circ -f_n(\bx')\geq {\phi(0)-\phi(\alpha_r)} \Big)<\delta,
        \end{equation}
        and $\PP^*(C_\phi^*(\eta^*, f_n)>C_\phi^*(\eta^*)+k_r)<\delta  $. The relation $\PP^*(C_\phi^*(\eta^*, f_n)>C_\phi^*(\eta^*)+k_r)<\delta$ implies
        \begin{equation}\label{eq:in_measure_necessary}
            \PP^*\Big( |\eta^* -1/2|\geq r, f_n\sgn(\eta^* -1/2)\leq \alpha_r\Big)<\delta
        \end{equation}
        due to \cref{lemma:C_phi_uniform}. Because $\phi$ is non-increasing, $\one_{f_n\leq 0}\leq \one_{\phi\circ f_n\geq \phi(0)}$ and since the function $z \mapsto \one_{z\geq \phi(0)}$ is upper semi-continuous,
    \[\int S_\e(\one_{f_n\leq  0}) d\PP_1\leq \int \one_{S_\e(\phi \circ f_n)\geq \phi(0)} d\PP_1=\int \one_{S_\e(\phi\circ f_n)(\bx)\geq \phi(0)}d\gamma_1^* =\gamma_1^*\big(S_\e(\phi\circ f_n)(\bx)\geq \phi(0)\big).\]
    Now \cref{eq:sup_approx_comp_slack_1} implies that for $n>N$, $S_\e(\phi\circ f_n)(\bx)< (\phi \circ f_n)(\bx')+\phi(0)-\phi(\alpha_r)$ outside a set of $\gamma_1^*$-measure $\delta$ and thus
    \begin{equation}\label{eq:upper_1}
        \int S_\e(\one_{f_n\leq 0})d\PP_1\leq \gamma_1^*\big(\phi\circ f_n (\bx')+\phi(0)-\phi(\alpha_r)> \phi(0)\big)+\delta\leq \PP_1^*\big(\phi\circ f_n> \phi(\alpha_r)\big) +\delta 
    \end{equation}

    Next, the monotonicity of $\phi$ implies that $\PP_1^*(\phi \circ f_n(\bx')> \phi(\alpha_r))\leq \PP_1^*(f_n<\alpha_r)$ and thus  
    \begin{align}
        &\int S_\e(\one_{f_n\leq 0})d\PP_1\leq \PP_1^*(f_n<\alpha_r)+\delta\leq \PP_1^*(f_n<\alpha_r, |\eta^* -1/2|\geq r)+2\delta. \label{eq:upper_2}
    \end{align}
    by \cref{eq:r_choice}. Next, \cref{eq:in_measure_necessary} implies $\PP_1^*(\eta^*\geq 1/2+r, f_n\leq \alpha_r)<\delta$ and consequently
    \begin{align*}
        &\int S_\e(\one_{f_n\leq 0})d\PP_1\leq \PP_1^*(f_n<\alpha_r, \eta^* \leq 1/2 -r)+3\delta \leq \PP_1^*( \eta^*\leq 1/2)+3\delta.
    \end{align*}
    Because $\delta$ is arbitrary, this relation implies \cref{eq:consistency_goal_1}. 
    Next, to prove \cref{eq:consistency_goal_0}, observe that $\one_{f\geq 0}=\one_{-f\leq 0}$, and thus the inequalities \cref{eq:upper_1} and \cref{eq:upper_2} hold with $-f_n$ in place of $f_n$, $\PP_0,$ $\PP_0^*$, $\gamma_0^*$ in place of $\PP_1,\PP_1^*$, $\gamma_1^*$, and \cref{eq:sup_approx_comp_slack_0} in place of \cref{eq:sup_approx_comp_slack_1}. \Cref{eq:in_measure_necessary} further implies $\PP_1^*(\eta^*\leq 1/2-r, f_n\geq -\alpha_r)<\delta$, and these relations imply \cref{eq:consistency_goal_0}.  
    \end{proof}
\section{Consistency Requires Uniqueness up to Degeneracy}\label{sec:consistency_implies_uniqueness}

   We prove the reverse direction of \cref{thm:uniqueness_and_consistency} by constructing a sequence of functions $f_n$ that minimize $\prm$ for which $R^\e(f_n)$ is constant in $n$ and not equal to the minimal adversarial Bayes risk.

        \begin{proposition}\label{prop:consistency_reverse}
        Assume that $\PP$ is absolutely continuous with respect to Lebesgue measure and that the adversarial Bayes classifier is not unique up to degeneracy. Then any consistent loss $\phi$ satisfying $C_\phi^*(1/2)=\phi(0)$ is not adversarially consistent. 
    \end{proposition}

    We will outline the proof in this section below, and the full argument is presented in \cref{app:consistency_reverse}. Before discussing the components of this argument, we illustrate this construction using the adversarial Bayes classifiers depicted in \cref{fig:gaussians_differing_means}. In this example, when $\e\leq (\mu_1-\mu_0)/2$, two adversarial Bayes classifiers are $A_1=\emptyset$ and $A_2=\Rset$, as depicted in \cref{fig:gaussians_differing_means_b,fig:gaussians_differing_means_c}. These sets are not equivalent up to degeneracy, and in particular the set $\{1\}$ is not an adversarial Bayes classifier. However, the fact that both $\emptyset$ and $\Rset$ are adversarial Bayes classifiers suggests that $\hat \eta(x) \equiv 1/2$ on all of $\Rset$, and thus $f(x)\equiv 0$ is a minimizer of $\prm$. Consequently, the function sequence
    \[f_n(x)=\begin{cases}
        \frac 1 n &\text{if }x\in \{1\}\\
        -\frac 1n&\text{otherwise}
    \end{cases}\]
    is a minimizing sequence of $\prm$ for which $R^\e(f_n)$ is constant in $n$ and not equal to the adversarial Bayes risk.

   To generalize this construction to an arbitrary distribution, one must find a set $\tilde A$ contained between two adversarial Bayes classifiers that is not itself an adversarial Bayes classifier. The absolute continuity assumption is key in this step of the argument. 
    \Cref{thm:adv_Bayes_construct} together with a result of \citep{Frank2024uniqueness} imply that when $\PP$ is absolutely continuous with respect to Lebesgue measure, the adversarial Bayes classifier is unique iff $\{\hat \eta >1/2\}$ and $\{\hat \eta \geq 1/2\}$ are equivalent up to degeneracy, see \cref{app:consistency_reverse} for a proof. 
   \begin{lemma}\label{lemma:not_degenerate_unique_adv_bayes}
        Assume $\PP$ is absolutely continuous with respect to Lebesgue measure. Then adversarial Bayes classifier is unique up to degeneracy iff the adversarial Bayes classifiers $\{\hat \eta >1/2\}$ and $\{\hat \eta \geq 1/2\}$ are equivalent up to degeneracy. 
    \end{lemma} 
   Consequently, if the adversarial Bayes classifier is not unique up to degeneracy, then there is a set $\tilde A$ that is not an adversarial Bayes classifier but $\{\hat \eta >1/2\}\subset \tilde A\subset \{\hat \eta \geq 1/2\}$. Next, we prove that one can replace the value of $\alpha_\phi(1/2)$ by 0 in the minimizer $\alpha_\phi(\hat \eta(\bx))$ and still retain a minimizer of $\prm$. Formally:
   \begin{lemma}\label{lemma:td_alpha_phi}
         Let $\alpha_\phi:[0,1]\to \Rset$ be as in \cref{lemma:smallest_minimizer_function} and define a function $\td \alpha_\phi :[0,1]\to \ov \Rset$ by 
         \begin{equation}\label{eq:tilde_alpha_phi}
            \td \alpha_\phi(\eta)=\begin{cases}
            \alpha_\phi(\eta)&\text{if }\eta \neq 1/2\\
            0&\text{otherwise}
        \end{cases}     
         \end{equation}
        
        Let $\hat \eta:\Rset^d\to [0,1]$ be the function described in \cref{prop:hat_eta}. If $\phi$ is consistent and $C_\phi^*(1/2)=\phi(0)$, then $\td \alpha(\hat \eta(\bx))$ is a minimizer of $R_\phi^\e$. 
    \end{lemma}
   See \cref{app:alpha_phi_zeros} for a proof of this result.
   Thus we select a sequence $f_n$ that is strictly positive on $\tilde A$, strictly negative on $\tilde A^C$, and approaches 0 on $\{\hat \eta =1/2\}$. Consider the sequence
        \begin{equation}\label{eq:f_n_def}
        f_n(\bx)=\begin{cases}
            \tilde\alpha_\phi(\hat \eta (\bx)) &\hat \eta(\bx)\neq 1/2\\
            \frac 1n & \hat \eta(\bx)=1/2, \bx \in \tilde A\\
            -\frac 1n & \hat \eta(\bx)=1/2, \bx \not \in \tilde A 
        \end{cases}
    \end{equation}
            Then $R^\e(f_n)=R^\e(\td A)>\inf_AR^\e(A)$ for all $n$ and one can show that $f_n$ is a minimizing sequence of $R_\phi^\e$. However, $f_n$ may assume the values $\pm \infty$ because the function $\alpha_\phi$ is $\ov \Rset$-valued. Truncating these functions produces an $\Rset$-valued sequence that minimizes $\prm$  but $R^\e(f_n)=R^\e(\tilde A)$ for all $n$, see \cref{app:smallest_minimizer_function} for details.

\section{Conclusion}
In summary, we prove that under a reasonable distributional assumption, a convex loss is adversarially consistent if and only if the adversarial Bayes classifier is unique up to degeneracy. This result connects an analytical property of the adversarial Bayes classifier to a statistical property of surrogate risks. Analyzing adversarial consistency in the multiclass setting remains an open problem. Prior work \citep{RobeyLatorreHassani2024nonZeroSum} has proposed an alternative loss function to address the issue of robust overfitting. Furthermore, many recent papers suggest that classification with a rejection option is a better framework for robust learning--- a non-exhaustive list includes \citep{ChenRaghuramChoi2022revisiting,ChenRaghuramChoi2023stratified,KatoCuiFukuhara20,shah2024calibratedlossesadversarialrobust}. Analyzing the behavior of these alternative risks using the tools developed in this paper is an open problem as well. Hopefully, our results will aid in the analysis and development of further algorithms for adversarial learning.

\newpage
\appendix

\section{Further Results on the Function $\hat \eta$--- Proof of \cref{thm:adv_Bayes_construct}}\label{app:adv_bayes_construct}

    We prove that the sets $\{\hat \eta >1/2\}$ and $\{\hat \eta\geq 1/2\}$ minimize $\prm$ by showing that their adversarial classification risks equal $\cdl(\PP_0^*,\PP_1^*)$, for the measures $\PP_0^*$, $\PP_1^*$ in \cref{prop:hat_eta}.
    \begin{proposition}\label{prop:hat_eta_adv_bayes_small}
        Let $\hat \eta$ be the function in \cref{thm:alpha_phi_hat_eta}. Then the sets $\{\hat \eta>1/2\}$, $\hat \eta \geq 1/2\}$ are both adversarial Bayes classifiers.
    \end{proposition}
    \begin{proof}
            We prove the statement for $\{\hat \eta >1/2\}$, the argument for the set $\{\hat \eta \geq 1/2\}$ is analogous. 
    
        Let $\PP_0^*, \PP_1^*$ be the measures of \cref{thm:hat_eta_adv_baye} and set $\PP^*=\PP_0^*+\PP_1^*$, $\eta^*=d\PP_1^*/d\PP^*$. Furthermore, let $\gamma_0^*$, $\gamma_1^*$ be the couplings between $\PP_0$, $\PP_0^*$ and $\PP_1$, $\PP_1^*$ supported on $\Delta_\e$. 
        
        First, \cref{it:I_S_hat_eta} implies that the function $\hat \eta(\bx)$ assumes its infimum on an ball $\ov{B_\e(\bx)}$ $\gamma_1^*$-a.e. and therefore $S_\e(\one_{\{\hat \eta >1/2\}^C})(\bx)=\one_{\{I_\e(\hat \eta)(\bx)>1/2\}^C}$ $\gamma_1^*$-a.e. (Recall the notation $I_\e$ was defined in \cref{eq:I_e_def}.) \Cref{it:I_S_hat_eta} further implies that $\one_{\{I_\e(\hat \eta)(\bx)>1/2\}^C}=\one_{\{\hat \eta(\bx')>1/2\}^C}$ $\gamma_1^*$-a.e. and consequently,
        \begin{equation}\label{eq:first_1_c_slack}
        S_\e(\one_{\{\hat \eta (\bx)>1/2\}^C})(\bx)=\one_{\{\hat \eta(\bx')>1/2\}^C}\quad \gamma_1^*\text{-a.e.}
        \end{equation}
        An analogous argument shows
        \begin{equation}\label{eq:first_0_c_slack}
            S_\e(\one_{\{\hat \eta>1/2\}})(\bx)=\one_{\{\hat \eta(\bx')>1/2\}}\quad \gamma_0^*\text{-a.e.}
        \end{equation}

        \cref{eq:first_1_c_slack,eq:first_0_c_slack} then imply that 
        \begin{equation*}
        \begin{aligned}
            \cprm(\{\hat \eta>1/2\})&=\int \one_{\{\hat \eta(\bx')>1/2\}^C}d\gamma_1^*+\int \one_{\{\hat \eta(\bx')>1/2\}}d\gamma_0^*\\ &=\int \one_{\{\hat \eta(\bx')>1/2\}^C}d\PP_1^*+\int \one_{\{\hat \eta(\bx')>1/2\}}d\PP_0^*=\int C(\eta^*,\one_{\{\hat \eta>1/2\}})d\PP^*.
        \end{aligned}
        \end{equation*}
        Next \cref{it:a.e._hat_eta} of \cref{prop:hat_eta} implies that $\hat \eta(\bx')=\eta^*(\bx')$ $\PP^*$-a.e. and consequently
        \[R^\e(\{\hat \eta>1/2\})=\int C(\eta^*,\one_{\{\eta^*>1/2\}})d\PP^*=\cdl(\PP_0^*,\PP_1^*).\]
        Therefore, the strong duality result in \cref{thm:strong_duality_classification} implies that $\{\hat \eta>1/2\}$ must minimize $\cprm$.
    \end{proof}

    Finally, the complementary slackness conditions from \cite[Theorem~2.4]{Frank2024uniqueness} characterize minimizers of $\cprm$ and maximizers of $\cdl$, and this characterization proves \cref{mal_maximal_1_uc,mal_maximal_0_uc}. Verifying these conditions would be another method of proving \cref{prop:hat_eta_adv_bayes_small}.

            \begin{theorem}\label{thm:complementary_slackness_classification_uc}
        The set $A$ is a minimizer of $R^\e$ and $(\PP_0^*,\PP_1^*)$ is a maximizer of $\bar R$ over the $W_\infty$ balls around $\PP_0$ and $\PP_1$ iff $W_\infty(\PP_0^*,\PP_0)\leq \e$, $W_\infty(\PP_1^*,\PP_1)\leq \e$, and 
    \begin{enumerate}[label=\arabic*)]
        
        \item

        \begin{equation}\label{eq:sup_comp_slack_classification_uc}
            \int S_\e(\one_{A^C})d\PP_1=\int \one_{A^C} d\PP_1^*\quad \text{and} \quad\int S_\e(\one_{A}) d\PP_0= \int \one_{ A} d\PP_0^* 
        \end{equation}
        \item \label{it:comp_slack_equation_classification_uc} 
    \begin{equation}\label{eq:complementary_slackness_necessary_classification_uc}
        C(\eta^*,\one_A(\bx'))=C^*(\eta^*(\bx'))\quad \PP^*\text{-a.e.}    
    \end{equation}
    where $\PP^*=\PP_0^*+\PP_1^*$ and $\eta^*=d\PP_1^*/d\PP^*$.
    \end{enumerate}

    \end{theorem}

    Let $\gamma_0^*$, $\gamma_1^*$ be couplings between $\PP_0$, $\PP_0^*$ and $\PP_1$, $\PP_1^*$ supported on $\Delta_\e$.
    Notice that because \cref{lemma:S_e_and_W_inf_uc} implies that $ S_\e(\one_{A^C})(\bx)\geq \one_{A^C}(\bx')$ $\gamma_1^*$-a.e. and $ S_\e(\one_A)(\bx)\geq \one_A(\bx')$, the complementary slackness condition in \cref{eq:sup_comp_slack_classification_uc} is equivalent to
    \begin{equation}\label{eq:other_comp_slack_consition}
        S_\e(\one_{A^C})(\bx)=\one_{A^C}(\bx')\quad \gamma_1^*\text{-a.e.}\quad\text{and}\quad S_\e(\one_A)(\bx)=\one_A(\bx')\quad \gamma_0^*\text{-a.e.}    
    \end{equation}

    This observation completes the proof of \cref{thm:adv_Bayes_construct}.

    \begin{proof}[Proof of \cref{thm:adv_Bayes_construct}]
     Let $\hat \eta$, $\PP_0^*$, $\PP_1^*$ be the function and measures of \cref{thm:hat_eta_adv_baye}, and set $\PP^*=\PP_0^*+\PP_1^*$, $\eta^*=d\PP_1^*/d\PP^*$. Let $\gamma_0^*$, $\gamma_1^*$ be couplings between $\PP_0$, $\PP_0^*$ and $\PP_1$, $\PP_1^*$ supported on $\Delta_\e$.
     
     \Cref{prop:hat_eta_adv_bayes_small} proves that the sets $\{\hat \eta >1/2\}$ and $\{\hat \eta \geq 1/2\}$ are in fact adversarial Bayes classifiers.
     If $A$ is any adversarial Bayes classifier, the complementary slackness condition \cref{eq:complementary_slackness_necessary_classification_uc} implies that $\one_{\eta^*>1/2}\leq \one_A\leq \one_{\eta^*\geq 1/2}$ $\PP^*$-a.e. Thus \cref{it:a.e._hat_eta} implies that  
        \[\one_{\{\hat \eta>1/2\}}(\bx')\leq \one_A(\bx') \leq \one_{\{\hat \eta \geq 1/2\}}(\bx')\quad \gamma_0^*\text{-a.e.}\]and 
        \[\one_{\{\hat \eta>1/2\}^C}(\bx')\leq \one_{A^C}(\bx') \leq \one_{\{\hat \eta \geq 1/2\}^C}(\bx') \quad \gamma_1^*\text{-a.e.}\]
        The complementary slackness condition \cref{eq:other_comp_slack_consition} then implies \cref{mal_maximal_1_uc,mal_maximal_0_uc}.
    \end{proof}

\section{Uniqueness up to Degeneracy and $\hat \eta(\bx)$--- Proof of \cref{thm:uniqueness_equivalences_1}}\label{app:uniqueness_equivalences_1}
    
    Theorem~3.4 of \citep{Frank2024uniqueness} proves the following result:
    \begin{theorem}\label{thm:uniqueness_equivalences_2}
    Assume that $\PP$ is absolutely continuous with respect to Lebesgue measure. Then the following are equivalent:
    \begin{enumerate}[label=\arabic*)]
            \item \label{it:unique_under_deg_2} The adversarial Bayes classifier is unique up to degeneracy
            \item \label{it:S_e_unique_uc_2} 
            Amongst all adversarial Bayes classifiers $A$, the value of $\int S_\e(\one_A)d\PP_0$ is unique or the value of $\int S_\e(\one_{A^C})d\PP_1$ is unique
    \end{enumerate}

\end{theorem}

Thus it remains to show that \cref{it:S_e_unique_uc_2} of \cref{thm:uniqueness_equivalences_2} is equivalent to \cref{it:eta_*_meas_zero} of \cref{thm:uniqueness_equivalences_1}. We will apply the complementary slackness conditions of \cref{thm:complementary_slackness_classification_uc}.
\begin{proof}[Proof of \cref{thm:uniqueness_equivalences_1}]

Let $\PP_0^*$, $\PP_1^*$ be the measures of \cref{prop:hat_eta}.

First, we show that \cref{it:S_e_unique_uc_2} implies \cref{it:eta_*_meas_zero}. Assume that \cref{it:S_e_unique_uc_2} holds. Notice that for an adversarial Bayes classifier $A$,
\[\int S_\e(\one_A)d\PP_0+ \int S_\e(\one_{A^C})d\PP_1 =R_*^\e\]
where $R_*^\e$ is the minimal value of $R^\e$. Thus amongst all adversarial Bayes classifiers $A$, the value of
$\int S_\e(\one_A)d\PP_0$ is unique iff the value of $\int S_\e(\one_{A^C})d\PP_1$ is unique. Thus \cref{it:S_e_unique_uc_2} implies both $\int S_\e(\one_{A_1})d\PP_0=\int S_\e(\one_{A_2})d\PP_0$ and $\int S_\e(\one_{A_1^C})d\PP_1=\int S_\e(\one_{A_2^C})d\PP_1$ for any two adversarial Bayes classifiers $A_1$ and $A_2$.

Applying this statement to the adversarial Bayes classifiers $\{\hat \eta >1/2\}$ and $\{\hat \eta \geq 1/2\}$ produces
\[\int S_\e(\one_{\{\hat \eta>1/2\}^C})d\PP_1 =\int S_\e(\one_{\{\hat \eta\geq 1/2\}^C})d\PP_1 \quad \text{and}\quad \int S_\e(\one_{\{\hat \eta>1/2\}})d\PP_0 =\int S_\e(\one_{\{\hat \eta\geq 1/2\}})d\PP_0
\]
The complementary slackness condition \cref{eq:sup_comp_slack_classification_uc} implies that
\[\int \one_{\{\hat \eta>1/2\}^C}d\PP_1^*=\int \one_{\{\hat \eta\geq 1/2\}^C}d\PP_1^*\quad \text{and}\quad \int \one_{\{\hat \eta>1/2\}}d\PP_0^*=\int \one_{\{\hat \eta\geq 1/2\}}d\PP_0^*\] 
and subsequently, \cref{it:a.e._hat_eta} of \cref{prop:hat_eta} implies that
\[\int \one_{\{\eta^*>1/2\}^C}d\PP_1^*=\int \one_{\{\eta^*\geq 1/2\}^C}d\PP_1^*\quad \text{and}\quad \int \one_{\{\eta^*>1/2\}}d\PP_0^*=\int \one_{\{\eta^*\geq 1/2\}}d\PP_0^*.\]
Consequently, $\PP^*(\eta^*=1/2)=0$.

To show the other direction, we apply the inequalities in \cref{thm:adv_Bayes_construct}. The complementary slackness conditions in \cref{thm:complementary_slackness_classification_uc} and the first inequality in \cref{thm:adv_Bayes_construct} imply that for any adversarial Bayes classifier $A$, 
\[\int \one_{\{\eta^*< 1/2\}}d\PP_1^*\leq \int S_\e(\one_{A^C})d\PP_1\leq \int \one_{\{\hat \eta^*\leq 1/2\}} d\PP_1^* \]

Consequently, if $\PP^*(\eta^*=1/2)=0$, then $\int \one_{\{\eta^*< 1/2\}}d\PP_1^*= \int S_\e(\one_{A^C})d\PP_1$, which implies that $\int S_\e(\one_{A^C})d\PP_1$ assumes a unique value over all possible adversarial Bayes classifiers.
\end{proof}

\section{Technical loss function proofs---Proof of \cref{lemma:smallest_minimizer_function,lemma:C_phi_uniform}}\label{app:consistent_losses}\label{app:smallest_minimizer_function}

To begin, we prove a result that compares minimizers of $C_\phi(\eta,\cdot)$ for differing values of $\eta$.

This result is then used to prove \cref{lemma:smallest_minimizer_function}.

\begin{lemma}\label{lemma:minimizers_comparison}
If $\alpha_2^*$ is any minimizer of $C_\phi(\eta_2,\cdot)$ and $\eta_2>\eta_1$, then $\alpha_\phi(\eta_1)\leq \alpha_2^*$.
\end{lemma}
\begin{proof}

    One can express $C_\phi(\eta_2,\alpha)$ as
    \[C_\phi(\eta_2,\alpha)=C_\phi(\eta_1,\alpha)+(\eta_2-\eta_1)(\phi(\alpha)-\phi(-\alpha))\]
    Notice that the function $\alpha \mapsto \phi(\alpha)-\phi(-\alpha)$ is non-increasing in $\alpha$. As $\alpha_\phi(\eta_1)$ is the smallest minimizer of $C_\phi(\eta_1,\cdot)$, if $\alpha<\alpha_\phi(\eta_1)$ then $C_\phi(\eta_1,\alpha)>C_\phi^*(\eta_1)$ and thus $C_\phi(\eta_2,\alpha)>C_\phi(\eta_2,\alpha_\phi(\eta_1))$. Consequently, every minimizer of $C_\phi(\eta_2,\cdot)$ must be greater than or equal to $\alpha_\phi(\eta_1)$.
\end{proof}

 Next we use this result to prove \cref{lemma:smallest_minimizer_function}.
\begin{proof}[Proof of \cref{lemma:smallest_minimizer_function}]
    For $\eta_2>\eta_1$, apply \cref{lemma:minimizers_comparison} with the choice $\alpha_2^*=\alpha_\phi(\eta_2)$.
\end{proof}

Next, if the loss $\phi$ is consistent, then $0$ can minimize $C_\phi(\eta,\cdot)$ only when $\eta=1/2$.

\begin{lemma}\label{lemma:consistent_zeros}
    Let $\phi$ be a consistent loss. If $0\in \argmin C_\phi(\eta,\cdot)$, then $\eta=1/2$. 
\end{lemma}
\begin{proof}
    Consider a distribution for which $\eta(\bx)\equiv \eta$ is constant. Then by the consistency of $\phi$, if $0$ minimizes $C_\phi(\eta,\cdot)$, then it also must minimize $C(\eta,\cdot)$ and therefore $\eta\leq 1/2$. 
        
        However, notice that $C_\phi(\eta,\alpha)=C_\phi(1-\eta,-\alpha)$. Thus if 0 minimizes $C_\phi(\eta,\cdot)$ it must also minimize $C_\phi(1-\eta,\cdot)$.
        The consistency of $\phi$ then implies that $1-\eta\leq 1/2$ as well and consequently, $\eta=1/2$.

\end{proof}

Combining these two results proves \cref{lemma:C_phi_uniform}.
 
\begin{proof}[Proof of \cref{lemma:C_phi_uniform}]
    Notice that $C_\phi(\eta,\alpha)=C_\phi(1-\eta,-\alpha)$ and thus it suffices to consider $\eta\geq 1/2+r$. 

    \Cref{lemma:consistent_zeros} implies that $C_\phi(1/2+r,\alpha_\phi(1/2+r))<\phi(0)$. Furthermore, as $\phi(-\alpha)\geq \phi(0)\geq \phi(\alpha)$ when $\alpha\geq 0$, one can conclude that $\phi(\alpha_\phi(1/2+r))<\phi(0)$. Now pick an $\alpha_r\in (0,\alpha_\phi(1/2+r))$ for which $\phi(\alpha_\phi(1/2+r))<\phi(\alpha_r)<\phi(0)$. Then by \cref{lemma:minimizers_comparison}, if $\eta\geq 1/2+r$, every $\alpha$ less than or equal to $\alpha_r$ does not minimize $C_\phi(\eta,\alpha)$ and thus $C_\phi(\eta,\alpha)-C_\phi^*(\eta)>0$. Now define
    \[k_r=\inf_{\substack{\eta \in [1/2+r,1]\\\alpha \in [-\infty, \alpha_r]}} C_\phi(\eta,\alpha)-C_\phi^*(\eta)\]
    The set $[1/2+r,1]\times [-\infty, \alpha_r]$ is sequentially compact and the function $(\eta,\alpha)\mapsto C_\phi(\eta,\alpha)-C_\phi^*(\eta)$ is continuous and strictly positive on this set. Therefore, the infimum above is assumed for some $\eta,\alpha$ and consequently $k_r>0$.

    Lastly, $\phi(\alpha_r)<\phi(0)$ implies $\alpha_r>0$.
\end{proof}
\section{Deferred arguments from \cref{sec:consistency_implies_uniqueness}--- Proof of \cref{prop:consistency_reverse}}\label{app:consistency_reverse}

     To start, we formally prove that if the adversarial Bayes classifier is not unique up to degeneracy, then the sets $\{\hat \eta >1/2\}$ and $\{\hat \eta \geq 1/2\}$ are not equivalent up to degeneracy.

    This result in Lemma~\ref{lemma:not_degenerate_unique_adv_bayes} relies on a characterization of equivalence up to degeneracy from \citep[Proposition~5.1]{Frank2024uniqueness}.

    \begin{theorem}\label{thm:equivalence_up_to_degeneracy_characterization}
        Assume that $\PP$ is absolutely continuous with respect to Lebesgue measure and let $A_1$ and $A_2$ be two adversarial Bayes classifiers. Then the following are equivalent:
        \begin{enumerate}[label=\arabic*)]
        \item\label{it:degeneracy_equiv_uc} The adversarial Bayes classifiers $A_1$ and $A_2$ are equivalent up to degeneracy 
        \item \label{it:S_1_uc}\label{it:S_0_uc} Either $S_\e(\one_{A_1})=S_\e(\one_{A_2})$-$\PP_0$-a.e. or $S_\e(\one_{A_2^C})=S_\e(\one_{A_1^C})$-$\PP_1$-a.e. 
        \end{enumerate}
    \end{theorem}
    Notice that when there is a single equivalence class, the equivalence between \cref{it:degeneracy_equiv_uc} and \cref{it:S_1_uc} of \cref{thm:equivalence_up_to_degeneracy_characterization} is simply the equivalence between \cref{it:unique_under_deg_2} and \cref{it:S_e_unique_uc_2} in \cref{thm:uniqueness_equivalences_2}. 
    This result together with \cref{thm:adv_Bayes_construct} proves \cref{lemma:not_degenerate_unique_adv_bayes}:

       \begin{proof}[Proof of \cref{lemma:not_degenerate_unique_adv_bayes}]
        Let $A$ be any adversarial Bayes classifier. If the adversarial Bayes classifiers $\{\hat \eta>1/2\}$ and $\{\hat \eta \geq 1/2\}$ are equivalent up to degeneracy, then \cref{thm:adv_Bayes_construct} and \cref{it:S_0_uc} of \cref{thm:equivalence_up_to_degeneracy_characterization} imply that $S_\e(\one_A)=S_\e(\one_{\{\hat \eta >1/2\}})$ $\PP_0$-a.e. \Cref{it:S_0_uc} of \cref{thm:equivalence_up_to_degeneracy_characterization} again implies that $A$ and $\{\hat \eta >1/2\}$ must be equivalent up to degeneracy, and consequently the adversarial Bayes classifier must be unique up to degeneracy.

        Conversely, if the adversarial Bayes classifier is unique up to degeneracy, then the adversarial Bayes classifiers $\{\hat \eta >1/2\}$ and $\{\hat \eta \geq 1/2\}$ must be equivalent up to degeneracy.
    \end{proof}

     Thus, if the adversarial Bayes classifier is not unique up to degeneracy, then there is a set $\td A$ with $\{\hat \eta >1/2\}\subset \td A\subset \{\hat \eta \geq 1/2\}$ that is not an adversarial Bayes classifier, and this set is used to construct the sequence $f_n$ in \cref{eq:f_n_def}. Next, we show that $f_n$ minimizes $\prm$ but not $\cprm$.


    \begin{proposition}\label{prop:reverse_ov_Rset}
        Assume that $\PP$ is absolutely continuous with respect to Lebesgue measure and that the adversarial Bayes classifier is not unique up to degeneracy. Then there is a sequence of $\ov \Rset$-valued functions $f_n$ that minimize $\prm$ but $R^\e(f_n)$ is constant in $n$ and not equal to the adversarial Bayes risk.
    \end{proposition}
    \begin{proof}

    By \cref{lemma:not_degenerate_unique_adv_bayes}, there is a set $\td A$ with $\{\hat \eta>1/2\}\subset \td A\subset \{\hat \eta \geq 1/2\}$ which is not an adversarial Bayes classifier. For this set $\td A$, define the sequence $f_n$ by \cref{eq:f_n_def}
    and let $\td \alpha_\phi$ be the function in \cref{lemma:td_alpha_phi}. \Cref{lemma:consistent_zeros} implies that $\tilde \alpha_\phi(\eta)\neq 0$ whenever $\eta\neq 1/2$ and thus $\{f_n>0\}=\tilde A$ for all $n$. We will show that in the limit $n\to \infty$, the function sequence $S_\e(\phi \circ f_n)$ is bounded above by $S_\e (\phi \circ \tilde \alpha_\phi(\hat \eta))$ while $S_\e(\phi \circ -f_n)$ is bounded above by $S_\e(\phi \circ -\td \alpha_\phi(\hat \eta))$. This result will imply that $f_n$ is a minimizing sequence of $R_\phi^\e$ due to \cref{lemma:td_alpha_phi}. 

    Let $\td S_\e(g)$ denote the supremum of a function $g$ on an $\e$-ball excluding the set $\hat \eta(\bx)=1/2$:
    \[ \td S_\e(g)(\bx)=\begin{cases}
        \sup_{\substack{\bx' \in \ov{B_\e(\bx)}\\ \hat \eta(\bx')\neq 1/2}} g(\bx') &\text{if } \ov{B_\e(\bx)}\cap \{\hat \eta \neq 1/2\}^C \neq \emptyset\\
        -\infty &\text{otherwise}
    \end{cases}\]
    With this notation, because $\td \alpha_\phi(1/2)=0$, one can express $S_\e(\phi \circ \td \alpha_\phi(\hat \eta))$, $S_\e(\phi \circ -\td \alpha_\phi(\hat \eta))$ as 
    \begin{equation}\label{eq:td_S_e_+}
        S_\e\big(\phi \circ \td \alpha_\phi(\hat \eta)\big)(\bx)=\begin{cases}
        \max\big( \td S_\e\big(\phi \circ \alpha_\phi(\hat \eta)\big)(\bx), \phi(0)\big) &\bx \in \{\hat \eta =1/2\}^\e \\
        S_\e\big(\phi \circ \alpha_\phi(\hat \eta)\big)(\bx)&\bx \not \in \{\hat \eta=1/2\}^\e
    \end{cases}
    \end{equation}

     \begin{equation}\label{eq:td_S_e_-}
        S_\e\big(\phi \circ -\td \alpha_\phi(\hat \eta)\big)(\bx)=\begin{cases}
        \max\big( \td S_\e\big(\phi \circ -\alpha_\phi(\hat \eta)\big)(\bx), \phi(0)\big) &\bx \in \{\hat \eta =1/2\}^\e \\
        S_\e\big(\phi \circ -\alpha_\phi(\hat \eta)\big)(\bx)&\bx \not \in \{\hat \eta=1/2\}^\e
    \end{cases}
    \end{equation}

    and similarly
    \begin{equation}\label{eq:td_S_e_+_n}
        S_\e(\phi \circ f_n)(\bx)\leq \begin{cases}
        \max\big( \td S_\e\big(\phi \circ \alpha_\phi(\hat \eta)\big)(\bx), \phi(-\frac 1n)\big) &\bx \in \{\hat \eta =1/2\}^\e \\
        S_\e(\phi \circ \alpha_\phi(\hat \eta))(\bx)&\bx \not \in \{\hat \eta=1/2\}^\e
    \end{cases}
    \end{equation}
    \begin{equation}\label{eq:td_S_e_-_n}
        S_\e(\phi \circ -f_n)(\bx)\leq \begin{cases}
        \max\big( \td S_\e\big(\phi \circ -\alpha_\phi(\hat \eta)\big)(\bx), \phi(-\frac 1n)\big) &\bx \in \{\hat \eta =1/2\}^\e \\
        S_\e\big(\phi \circ -\alpha_\phi(\hat \eta)\big)(\bx)&\bx \not \in \{\hat \eta=1/2\}^\e
    \end{cases}
    \end{equation}
    Therefore, by comparing \cref{eq:td_S_e_+_n} with \cref{eq:td_S_e_+} and \cref{eq:td_S_e_-_n} with \cref{eq:td_S_e_-}, one can conclude that 
    \begin{equation}\label{eq:ls_to_min_risk}
    \limsup_{n\to \infty} S_\e(\phi \circ f_n)\leq S_\e\big(\phi \circ \td \alpha_\phi(\hat \eta)\big)\quad \text{and}\quad \limsup_{n\to \infty} S_\e(\phi \circ -f_n)\leq S_\e\big(\phi \circ -\td \alpha_\phi(\hat \eta)\big).
    \end{equation}
    Furthermore, the right-hand of \cref{eq:td_S_e_+_n} is bounded above by $ S_\e(\phi \circ \alpha_\phi(\hat \eta))+\phi(-1)$ while the right-hand of \cref{eq:td_S_e_-_n} is bounded above by $ S_\e(\phi \circ -\alpha_\phi(\hat \eta))+\phi(-1)$. Thus the dominated convergence theorem and \cref{eq:ls_to_min_risk} implies that 
    \[\limsup_{n\to \infty} R_\phi^\e(f_n)\leq R_\phi^\e(\td \alpha_\phi(\hat \eta))\] and thus $f_n$ minimizes $R_\phi^\e$.

    \end{proof}

    Lastly, it remains to construct an $\Rset$-valued sequence that minimizes $\prm$ but not $\cprm$. To construct this sequence, we threshhold a subsequence $f_{n_j}$ of $f_n$ at an appropriate value $T_j$. If $g$ is an $\ov \Rset$-valued function and $g^{(N)}$ is the function $g$ threshholded at $N$, then $\lim_{N\to \infty } R_\phi^\e(g^{(N)})=R_\phi^\e(g)$.
    \begin{lemma}\label{lemma:threshholding}
        Let $g$ be an $\ov \Rset$-valued function and let $g^{(N)}=\min(\max(g,-N),N)$. Then $\lim_{N\to \infty} \prm(g^{(N)})=\prm(g)$.
    \end{lemma}
    See \cref{app:threshholding} for a proof. \Cref{prop:consistency_reverse} then follows from this lemma and \cref{prop:reverse_ov_Rset}:
    \begin{proof}[Proof of \cref{prop:consistency_reverse}]
        Let $f_n$ be the $\ov\Rset$-valued sequence of functions in \cref{prop:reverse_ov_Rset}, and let $f_{n_j}$ be a subsequence for which $\prm(f_{n_j})-\inf_f R_\phi^\e(f)<1/j$. Next, \cref{lemma:threshholding} implies that for each $j$ one can pick a threshhold $N_j$ for which $|\prm(f_{n_j})-\prm(f_{n_j}^{(N_j)})|\leq 1/j$. Consequently, $f^{(N_j)}_{n_j}$ is an $\Rset$-valued sequence of functions that minimizes $\prm$. However, notice that $\{f\leq 0\}=\{f^{(T)}\leq 0\}$ and $\{f>0\}=\{f^{(T)}>0\}$ for any strictly positive threshhold $T$. Thus $\cprm(f_{n_j}^{(N_j)})=\cprm(f_{n_j})$ and consequently $f_{n_j}^{(N_j)}$ does not minimize $\cprm$.  
    \end{proof}
    
    \subsection{Proof of \cref{lemma:td_alpha_phi}}\label{app:alpha_phi_zeros}
        The proof of \cref{lemma:td_alpha_phi} follows the same outline as the argument for \cref{prop:hat_eta_adv_bayes_small}: we show that $\prm(\tilde\alpha_\phi(\hat \eta))=\dl(\PP_0^*,\PP_1^*)$ for the measures $\PP_0^*$, $\PP_1^*$ in \cref{prop:hat_eta}, and then \cref{thm:strong_duality_surrogate} implies that $\tilde \alpha_\phi(\hat \eta)$ must minimize $\prm$. Similar to the proof of \cref{prop:hat_eta_adv_bayes_small}, swapping the order of the $S_\e$ operation and $\tilde \alpha_\phi$ is a key step. To show that this swap is possible, we first prove that $\tilde \alpha_\phi$ is monotonic.

    \begin{lemma}\label{lemma:td_alpha_monotonic}
        If $C_\phi^*(1/2)=\phi(0)$, then the function $\tilde \alpha_\phi:[0,1]\to \ov \Rset$ defined in \cref{eq:tilde_alpha_phi} is non-decreasing and maps each $\eta$ to a minimizer of $C_\phi(\eta,\cdot)$.
    \end{lemma}
    \begin{proof}
        \Cref{lemma:smallest_minimizer_function} implies that $\td \alpha_\phi(\eta)$ is a minimizer of $C_\phi(\eta,\cdot)$ for all $\eta\neq 1/2$ and the assumption $C_\phi^*(1/2)=\phi(0)$ implies that $\td \alpha_\phi(1/2)$ is a minimizer of $C_\phi(1/2,\cdot)$.
        Furthermore, \cref{lemma:smallest_minimizer_function} implies that $\tilde \alpha_\phi$ is non-decreasing on $[0,1/2)$ and $(1/2,1]$. However, \cref{lemma:C_phi_uniform} implies that $\alpha_\phi(\eta)<0$ when $\eta\in [0,1/2)$ and $\alpha_\phi(\eta)>0$ when $\eta \in (1/2,1]$. Consequently, $\tilde \alpha_\phi$ is non-decreasing on all of $[0,1]$.
    \end{proof}
    This result together with the properties of $\PP_0^*$, $\PP_1^*$ in \cref{thm:hat_eta_adv_baye} suffice to prove \cref{lemma:td_alpha_phi}.
    \begin{proof}[Proof of \cref{lemma:td_alpha_phi}]
        Let $\PP_0^*, \PP_1^*$ be the measures of \cref{thm:hat_eta_adv_baye} and set $\PP^*=\PP_0^*+\PP_1^*$, $\eta^*=d\PP_1^*/d\PP^*$. We will prove that $\prm(\tilde \alpha_\phi(\hat \eta))=\dl(\PP_0^*,\PP_1^*)$ and thus \cref{thm:strong_duality_surrogate} will imply that $\tilde \alpha_\phi(\hat \eta)$ minimizes $\prm$.
        Let $\gamma_0^*$ and $\gamma_1^*$ be the couplings supported on $\Delta_\e$ between $\PP_0$, $\PP_0^*$ and $\PP_1$, $\PP_1^*$ respectively. \Cref{it:I_S_hat_eta} of \cref{prop:hat_eta} and \cref{lemma:td_alpha_monotonic} imply that 
        \[S_\e(\phi(\td \alpha_\phi(\hat \eta)))(\bx)= \phi(\td \alpha_\phi(I_\e(\hat \eta(\bx)))) = \phi(\tilde \alpha_\phi(\hat \eta(\bx')))\quad \gamma_1^*\text{-a.e.} \]
        and 
        \[S_\e(\phi(-\td \alpha_\phi(\hat \eta)))(\bx)= \phi(-\td \alpha_\phi(S_\e(\hat \eta(\bx)))) = \phi(\tilde \alpha_\phi(-\hat \eta(\bx')))\quad \gamma_0^*\text{-a.e.} \]
        (Recall the the notation $I_\e$ was introduced in \cref{eq:I_e_def}.)
        Therefore,
        \begin{equation*}
            \begin{aligned}
                \prm(\tilde \alpha_\phi(\hat \eta))&= \int \phi(\tilde \alpha_\phi(\hat \eta(\bx')) d\gamma_1^*+ \int \phi(-\tilde \alpha_\phi(\hat \eta(\bx')))d\gamma_0^*\\
                &= \int \phi(\tilde \alpha_\phi(\hat \eta(\bx'))) d\PP_1^*+ \int \phi(-\tilde \alpha_\phi(\hat \eta(\bx')))d\PP_0^*=\int C_\phi(\eta^*,\tilde \alpha_\phi(\hat \eta))d\PP^*
            \end{aligned}    
        \end{equation*}

        Next, \cref{it:a.e._hat_eta} of \cref{prop:hat_eta} implies that $\hat \eta(\bx')=\eta^*(\bx')$  and consequently 
        \[\prm(\tilde \alpha_\phi(\hat \eta))=\int C_\phi(\eta^*,\tilde \alpha_\phi(\hat \eta))d\PP^*=\int C_\phi(\eta^*,\tilde \alpha_\phi(\eta^*))d\PP^*=\int C_\phi^*(\eta^*)d\PP^*=\dl(\PP_0^*,\PP_1^*)\]
        Therefore, the strong duality result in \cref{thm:strong_duality_surrogate} implies that $\tilde \alpha_\phi(\hat \eta)$ must minimize $\prm$.

    \end{proof}
    \subsection{Proof of \cref{lemma:threshholding}}\label{app:threshholding}
    This argument is taken from the proof of Lemma~8 in \cite{FrankNilesWeed23consistency}.
    \begin{proof}[Proof of \cref{lemma:threshholding}]
        Define
        \[\sigma_{[a,b]}(\alpha)=\begin{cases} a&\text{if }\alpha <a\\
        \alpha&\text{if }\alpha\in [a,b]\\
        b&\text{if } \alpha> b
        \end{cases}\]
        Notice that 
        \[S_\e(\sigma_{[a,b]}(h))=\sigma_{[a,b]}(S_\e(h))\]
        and 
        \[\phi(\sigma_{[a,b]}(g))=\sigma_{[\phi(b),\phi(a)]}(\phi(g))\]
        for any functions $g$ and $h$.
        Therefore,
        \[S_\e(\phi(g^{(N)}))=\sigma_{[\phi(N),\phi(-N)]}(S_\e(\phi \circ g))\quad \text{and} \quad S_\e(\phi \circ -g^{(N)})=\sigma_{[\phi(N),\phi(-N)]}(S_\e(\phi \circ -g)),\] which converge to $S_\e(\phi\circ g)$ and $S_\e(\phi \circ -g)$ pointwise and $N\to \infty$. Furthermore, the functions $S_\e(\phi \circ g^{(N)})$ and $S_\e(\phi\circ -g^{(N)})$ are bounded above by 
        \[S_\e(\phi\circ g^{(N)})\leq S_\e(\phi\circ g)+\phi(1)\quad\text{and}\quad S_\e(\phi\circ -g^{(N)})\leq S_\e(\phi\circ -g)+\phi(1)\] for $N\geq 1$.
        As the functions $S_\e(\phi\circ g)+\phi(1)$ and $S_\e(\phi\circ- g)+\phi(1)$ are integrable with respect to $\PP_1$ and $\PP_0$ respectively, the dominated convergence theorem implies that 
        \[\lim_{n\to \infty} R_\phi^\e(g^{(N)})= R_\phi^\e(g).\]
        
    \end{proof}

\bibliographystyle{abbrvnat}
\bibliography{bibliography,bib_other,bibliography_mohri_anqi_yutao_H_consistency}
\end{document}